\documentclass[letterpaper, 10 pt, conference]{ieeeconf}  %

\IEEEoverridecommandlockouts                              %

\title{\LARGE \bf
Think Deep and Fast: Learning Neural Nonlinear Opinion Dynamics from Inverse Dynamic Games for Split-Second Interactions
}

\author{Haimin Hu$^{1,2,*}$, Jaime Fernández Fisac$^{1}$, Naomi Ehrich Leonard$^{3}$, \\ Deepak Gopinath$^{2}$, Jonathan DeCastro$^{2}$, and Guy Rosman$^{2}$%
\thanks{$^{1}$Department of Electrical and Computer Engineering, Princeton University, NJ, USA,
        {\tt\small \{haiminh,jfisac\}@princeton.edu}}%
\thanks{$^{2}$Toyota Research Institute, MA, USA,
        {\tt\small \{jonathan.decastro,
        deepak.gopinath,guy.rosman\}@tri.global}}%
\thanks{$^{3}$Department of Mechanical and Aerospace Engineering, Princeton University, NJ, USA,
        {\tt\small naomi@princeton.edu}}
\thanks{*Work conducted while HH was an intern at Toyota Research Institute.}%
\thanks{This work is supported by Toyota Research Institute (TRI). It solely reflects the opinions and conclusions of its authors and not TRI, or any other Toyota entity.
The authors thank Jingqi Li and Xinjie Liu for very helpful discussions on inverse dynamic games, and Alessio Franci for generously providing insights into nonlinear opinion dynamics.}
}

\usepackage{multicol}

\usepackage{float}
\usepackage{amsfonts}
\usepackage{comment}
\usepackage{times}
\usepackage{amsmath}
\usepackage{changepage}
\usepackage{amssymb}
\usepackage{graphicx}
\usepackage{adjustbox}
\usepackage{latexsym}
\usepackage{enumerate}
\usepackage[font=small,belowskip=0pt,labelfont=bf]{caption}
\usepackage{mathtools}
\usepackage{algorithm}
\usepackage{algcompatible}
\usepackage{algpseudocode}
\usepackage{bbm}
\usepackage{xcolor}
\usepackage{bm}
\usepackage{booktabs}
\usepackage[xindy, acronyms, nowarn]{glossaries}

\usepackage{url}
\usepackage[bookmarks=true]{hyperref}
 \hypersetup{
     colorlinks=true,
     linkcolor=blue,
     filecolor=blue,
     citecolor=blue,      
     urlcolor=black,
     }

\usepackage[%
    style=numeric-comp,
    sorting=none,
    backend=biber,
    sortcites=true,
    doi=false,
    url=false,
    firstinits=true,
    hyperref,
    isbn=false,
    eprint=false,
    minbibnames=2,   %
    maxbibnames=6,
    block=none]
    {biblatex}

\bibliography{references}

\newbibmacro{string+doiurl}[1]{
    \iffieldundef{doi}{
        \iffieldundef{url}{#1}{\href{\thefield{url}}{#1} }
    }{\href{https://dx.doi.org/\thefield{doi}}{#1}}
}
\DeclareFieldFormat{title}{\usebibmacro{string+doiurl}{\mkbibemph{#1}}}
\DeclareFieldFormat[article,incollection,techreport,inproceedings,book,misc]%
    {title}%
    {\usebibmacro{string+doiurl}{\mkbibquote{#1}}}

\renewbibmacro{in:}{}
\AtEveryBibitem{%
    \clearlist{language}%
    \clearfield{isbn}%
    \clearfield{issn}
}

\usepackage{balance}

\usepackage{amsmath}
\usepackage{bm}

\newtheorem{theorem}{Theorem}
\newtheorem{lemma}{Lemma}    

\newtheorem{remark}{Remark}
\newtheorem{example}{Example}

\usepackage{ifthen}
\newboolean{include-notes}
\newboolean{include-new}
\newboolean{include-remove}
\setboolean{include-notes}{true}
\setboolean{include-new}{false}
\setboolean{include-remove}{false}

\usepackage{xcolor}
\usepackage[normalem]{ulem}
\newcommand{\jaime}[1]{\ifthenelse{\boolean{include-notes}}{\textcolor{orange}{\textbf{Jaime:} #1}}{}}
\newcommand{\haimin}[1]{\ifthenelse{\boolean{include-notes}}{\textcolor{magenta}{\textbf{Haimin:} #1}}{}}

\newcommand{\todo}[1]{\ifthenelse{\boolean{include-notes}}{\textcolor{teal}{\textbf{TODO:} #1}}{}}
\newcommand{\remove}[1]{\ifthenelse{\boolean{include-remove}}{\textcolor{red}{\sout{#1}}}{}}
\newcommand{\new}[1]{\ifthenelse{\boolean{include-new}}{\textcolor{purple}{#1}}{#1}}

\usepackage{lipsum}

\DeclareDocumentEnvironment{example}{}{\noindent\textbf{Running example:}\itshape}{}
\newcommand{\p}[1]{\smallskip \noindent \textbf{{#1}.}}

\glsdisablehyper
\makeglossaries
\newglossaryentry{NOD}
{
  name={NOD},
  plural={NOD},
  description={nonlinear opinion dynamics},
  first={nonlinear opinion dynamics (\glsentrytext{NOD})},
  descriptionplural={nonlinear opinion dynamics},
  firstplural={nonlinear opinion dynamics (\glsentryplural{NOD})}
}

\newglossaryentry{DNN}
{
  name={DNN},
  description={deep neural network},
  first={deep neural network (\glsentrytext{DNN})},
}

\newglossaryentry{MLE}
{
  name={MLE},
  description={maximum likelihood estimation},
  first={maximum likelihood estimation (\glsentrytext{MLE})},
}

\newglossaryentry{ODG}
{
  name={ODG},
  description={opinion-guided dynamic game},
  first={opinion-guided dynamic game (\glsentrytext{ODG})},
}

\newglossaryentry{MLP}
{
  name={MLP},
  description={multilayer perceptron},
  first={multilayer perceptron (\glsentrytext{MLP})},
}

\newglossaryentry{RL}
{
  name={RL},
  description={reinforcement learning},
  first={reinforcement learning (\glsentrytext{RL})},
}

\newglossaryentry{E2E-BC}
{
  name={E2E-BC},
  description={end-to-end behavior cloning},
  first={End-to-end behavior cloning (\glsentrytext{E2E-BC})},
}

\newcommand{\real}{\operatorname{Re}}

\newcommand{\reals}{\mathbb{R}}

\newcommand{\distr}{p}

\newcommand{\grad}{{\nabla}}

\newcommand{\jacobian}{{\mathbf{J}}}

\DeclareMathOperator{\column}{col}

\newcommand{\diag}{\operatorname{diag}}
\newcommand{\blkdiag}{\operatorname{blkdiag}}

\newcommand{\opinionDyn}{{g}}

\newcommand{\opnstate}{{z}}

\newcommand{\attstate}{{\lambda}}

\newcommand{\bias}{{b}}

\newcommand{\damping}{{d}}
\newcommand{\saturation}{{S}}

\newcommand{\opnidx}{{\ell}}
\newcommand{\opnidxaux}{{p}}

\newcommand{\nodparam}{\eta}

\newcommand{\iagent}{{i}}
\newcommand{\nagents}{{N_a}}

\newcommand{\noptioni}{{N_{o^\iagent}}}

\newcommand{\iagentaux}{{j}}

\newcommand{\state}{{x}}

\newcommand{\ctrl}{{u}}

\newcommand{\obs}{{y}}
\newcommand{\traj}{{\mathbf{x}}}%
\newcommand{\csig}{{\mathbf{u}}}

\newcommand{\obstraj}{\mathbf{\obs}}

\newcommand{\dyn}{{f}}

\newcommand{\policy}{{\pi}}

\newcommand{\orderset}{\mathcal{I}}
\newcommand{\ordersetagent}{\orderset_{a}}
\newcommand{\ordersetoptioni}{\orderset_{o^i}}

\newcommand{\ordersetthetai}{\orderset_{\theta_i}}

\newcommand{\stagecost}{c}
\newcommand{\eqset}{\Gamma}

\newcommand{\nnparam}{\phi}
\newcommand{\loss}{L}

\newcommand{\nnnod}{h_\nodparam}
\newcommand{\nnnodinit}{h_{\opnstate_0}}

\bibliography{references}

\begin{document}

\maketitle
\thispagestyle{empty}
\pagestyle{empty}

\begin{abstract}
Non-cooperative interactions commonly occur in multi-agent scenarios such as car racing, where an ego vehicle can choose to overtake the rival, or stay behind it until a safe overtaking ``corridor'' opens.
While an expert human can do well at making such time-sensitive decisions, 
autonomous agents are incapable of rapidly reasoning about complex, potentially conflicting options, leading to suboptimal behaviors such as deadlocks.
Recently, the \gls{NOD} model has proven to exhibit fast opinion formation and avoidance of decision deadlocks.
However, \gls{NOD} modeling parameters are oftentimes assumed fixed, limiting their applicability in complex and dynamic environments.  It remains an open challenge to determine such parameters \textit{automatically and adaptively}, accounting for the ever-changing environment.
In this work, we propose for the first time a learning-based and game-theoretic approach to synthesize a Neural \gls{NOD} model from expert demonstrations, given as a dataset containing (possibly incomplete) state and action trajectories of interacting agents.
We demonstrate Neural \gls{NOD}'s ability to make fast and deadlock-free decisions in a simulated autonomous racing example.
We find that Neural \gls{NOD} consistently outperforms the state-of-the-art data-driven inverse game baseline in terms of safety and overtaking performance.
\end{abstract}

\begin{figure}[!hbtp]
  \centering
  \includegraphics[width=1.0\columnwidth]{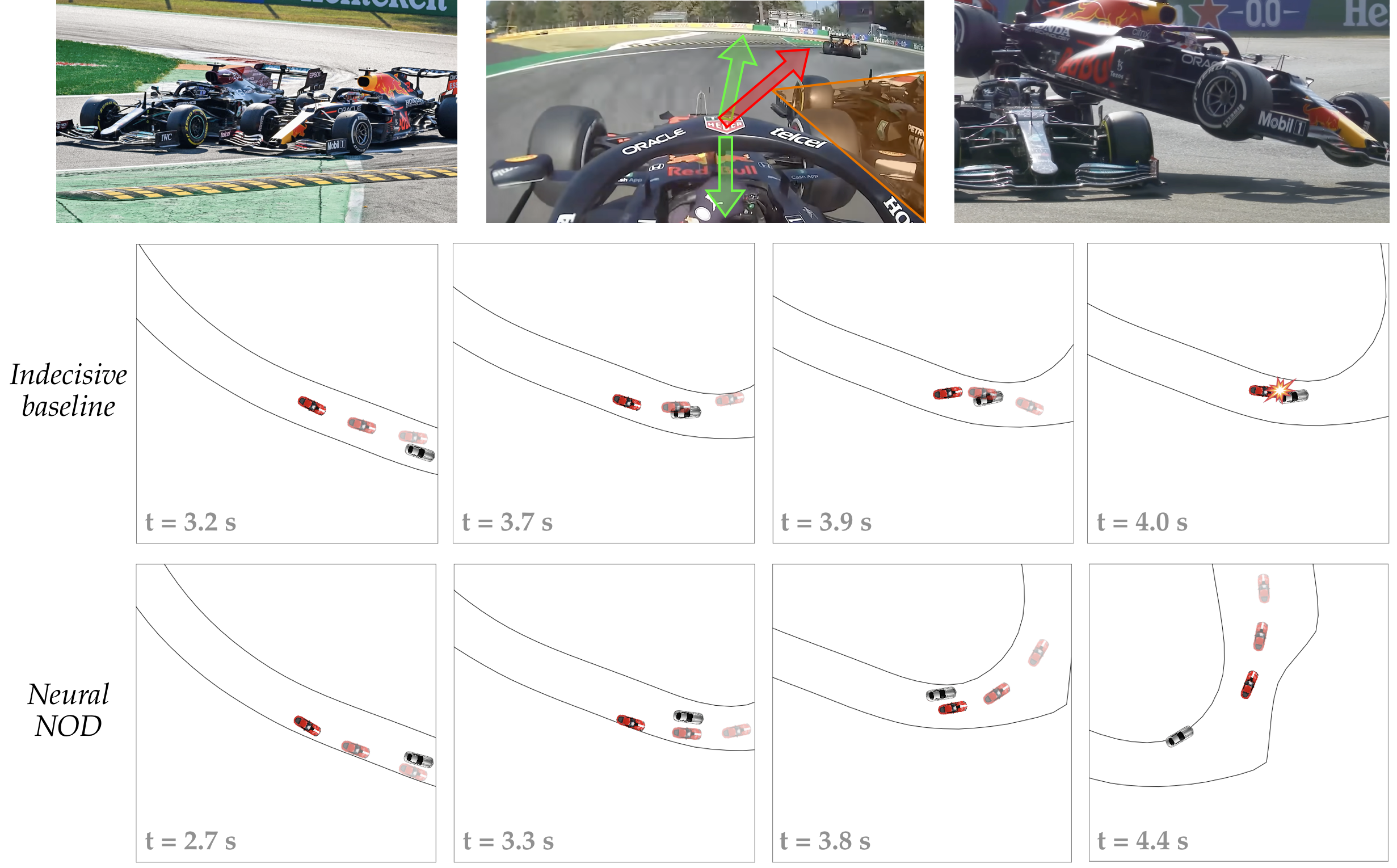}
  \caption{\label{fig:front_fig} 
  {\footnotesize
  Rapid and resolute decision-making is essential for non-cooperative multi-agent interactions like car racing.
  \emph{Top:} During the 2021 Formula~1 Italian Grand Prix, a fatal collision occurred involving championship contenders Max Verstappen and Lewis Hamilton.
  Verstappen was deemed predominantly responsible because, despite the overtaking opportunity closing after Hamilton (orange triangle) led him into the corner, he had options to avoid the collision by slowing down or taking the emergency alternative route (green arrows), but he failed to make a \textit{timely decision}, continuing along the racing line (red arrow) and ultimately leading to an inevitable collision later on.
  \emph{Middle:} A similar scenario arises in simulated autonomous racing when the ego car (red) uses an indecisive policy, hesitating between overtaking the rival (silver) from the inside or outside of the corner (as seen in its planned future motions depicted in transparent snapshots), ultimately resulting in a collision.
  \emph{Bottom:} The proposed Neural \gls{NOD} model reasons \textit{split-second} strategic interactions between the agents, rendering safe and decisive overtaking maneuvers.
  }
  }
  \vspace{-7mm}
\end{figure}

\section{Introduction}
\label{sec:intro}

As autonomous multi-agent systems evolve toward an unprecedented complexity, especially in high-stakes scenarios such as car racing, the necessity for reliable decision-making in real-time becomes paramount.
Rapid decision-making is critical in these settings, not only for performance but also for ensuring safety during complex, close-proximity interactions.
For example, in autonomous car racing (\autoref{fig:front_fig}), the ego vehicle must decide, in a split second, between overtaking a rival car without crashing into the rival or strategically trailing it to capitalize on emerging opportunities.

Game-theoretic motion planning techniques explicitly reason about the coupled interests among agents, and have been widely used for multi-agent non-cooperative interactions, e.g., autonomous driving~\cite{schwarting2019social,wang2021game}, human--robot interaction~\cite{li2019differential,bajcsy2021analyzing,hu2023activeIJRR}, and distributed control systems~\cite{maestre2011distributed,hu2020non,williams2023distributed}.
Computing equilibria for general dynamic games is oftentimes computationally intractable due to the inherent non-convexity involving, e.g., nonlinear system dynamics, non-convex cost functions, and collision-avoidance constraints.
To tackle such computation challenges, Fridovich-Keil et al.~\cite{fridovich2020efficient} apply iterative linear-quadratic (ILQ) approximations to solve general-sum dynamic games, enabling real-time computation of approximate Nash equilibria.
Recent work~\cite{lidard2024blending} extends this approach to the stochastic setting by blending a game policy (with a fixed payoff) with a data-driven reference policy.
While these game-theoretic approaches can effectively capture non-cooperative interactions, they often fall short in rapid calibration of strategies to account for agents' evolving intents.

As an alternative to model-based optimization, data-driven methods have also shown competitive performance in interactive autonomy.
GT Sophy~\cite{wurman2022outracing} demonstrates that a well-trained deep \gls{RL} policy can win a head-to-head competition against some of the world’s best drivers in Gran Turismo, a popular car racing video game.
However, the policy does not explicitly reason about the \emph{strategic interactions} among players.
To investigate this gap, Chen et al.~\cite{chen2023learn} combine \gls{RL} with self-supervised learning that models the opponent behaviors and show that the resulting policy outperforms methods with manually designed opponent strategies in multi-agent car racing.
Recent work~\cite{hu2023belgame} leverages belief-space safety analysis to actively reduce opponent uncertainty in adversarial interactions.
Despite promising performance and scalability, these black-box methods may lack \textit{strategic} reasoning when deployed for complex interactions and can struggle with practical limitations such as poor generalization, particularly when faced with insufficient or unrepresentative data~\cite{ren2021generalization}.

The recently developed \gls{NOD} model~\cite{bizyaeva2022nonlinear} provides a principled way to make deadlock-free decisions in multi-agent interactions.
Opinion states numerically represent agents' agreement or disagreement on a decision.
\gls{NOD} achieves rapid opinion formation through a nonlinear bifurcation. A local bifurcation is a change in the number and/or stability of equilibrium solutions as a parameter varies through a bifurcation point, which corresponds to a singularity in the dynamics.  Near this point, the process is \textit{ultrasensitive} to external inputs, e.g., physical interactions among agents, and the implicit threshold on the input associated with rapid opinion formation can be tuned by modulating parameters~\cite{leonard2024fast}.

Cathcart et al.~\cite{cathcart2022opinion} use \gls{NOD} to break social deadlocks in a human--robot navigation problem.
Amorim et al.~\cite{amorim2024threshold} propose an \gls{NOD}-based decision-making framework for robots choosing between different spatial tasks while adapting to environmental conditions, where opinions represent agents' preference for a particular task.
Paine et al.~\cite{Paine2024} leverage \gls{NOD} to make group decisions for autonomous multi-robot systems, demonstrating robust and interpretable collective behaviors in a field test involving multiple unmanned surface vessels.
Hu et al.~\cite{hu2023emergent} design for the first time a general algorithm to automatically synthesize an \gls{NOD} model.
They propose to construct \gls{NOD} parameters based on dynamic game value functions, thus making the opinion evolution dependent on agents' physical states.
However, this approach assumes that agents face a set of mutually exclusive options, and it cannot effectively integrate data-driven prior knowledge to facilitate online decision-making.
For a comprehensive review of \gls{NOD}, we refer the readers to Leonard et al.~\cite{leonard2024fast}.
For general multi-agent non-cooperative interaction, it remains an open challenge to select suitable \gls{NOD} model parameters such that the ego agent can reason how ongoing interactions may drive rapid changes in its intended course of action as well as those of other agents.

\p{Contributions}
We propose for the first time a learning-based and game-theoretic approach to synthesize a Neural \gls{NOD} model for time-sensitive multi-agent decision-making.
The model can be trained over a dataset of diverse interaction trajectories, can account for general opinion spaces, and can be plug-and-play used by any model-based game solver for automatic cost tuning.
Compared to previous \gls{NOD} with static model parameters, the Neural \gls{NOD} dynamically adjusts its parameters according to the evolving \textit{\emph{physical}} states, allowing the ego robot to rapidly adapt its decision to the fast-changing environment.
We provide verifiable conditions under which the model is guaranteed to avoid indecision.
We deploy the Neural \gls{NOD} learned from real human racing data for simulated car racing and compare our approach with the state-of-the-art data-driven game-theoretic planning baseline.

\section{Preliminaries}
\label{sec:prelim}

\p{Notation}
We use superscript $\iagent$ to indicate agent $\iagent \in \ordersetagent := \{1, 2, \ldots, \nagents\}$.
We define $[N] := \{1,2,\ldots,N\}$.
Given a function $f$, we denote $\jacobian(f(\cdot))$ as its Jacobian, and $\grad_\theta f$ as the derivative of $f$ with respect to variable $\theta$.
Given a matrix $A \in \reals^{n \times n}$,
let $\real^+(\sigma(A))$ be the set of positive real part of eigenvalues of $A$.
Define $\column(\cdot)$ as the operator that stacks all the arguments into a column vector.

\p{General-Sum Dynamic Games}
We consider an $\nagents$-player finite-horizon, discrete-time dynamic game governed by a nonlinear dynamical system:
\begin{equation}
\label{eq:dyn_sys}
\state_{t+1} = \dyn(\state_t, \ctrl_t),
\end{equation}
where $\state \in \reals^{n_x}$ is the \emph{joint} state of the system, which we assume to be fully observable at runtime, $\ctrl_t:=(\ctrl^{1}_t, \ldots, \ctrl^{\nagents}_t)$ where $\ctrl^\iagent \in \reals^{m_\iagent}$ is the control of player $\iagent$.
The objective of each player $\iagent$ is to minimize a cost functional:
\begin{equation}
\label{eq:game_cost}
    J^\iagent\left(\policy; \theta^\iagent \right) := \textstyle\sum_{k=0}^T \stagecost_k^\iagent(\state_k, \ctrl_k; \theta^\iagent),
\end{equation}
where $\stagecost^\iagent_k(\cdot)$ is the cost of stage $k$, $\theta^i \in \reals^{n_{\theta^i}}$ is the (possibly unknown) parameter of the stage cost, and $\policy := (\policy^1(\cdot),\ldots,\policy^\nagents(\cdot))$ is a tuple containing all players' control policies, which determines the \emph{information structure}~\cite[Ch.~5]{bacsar1998dynamic} of the game, e.g., open-loop if $\ctrl^\iagent_t = \policy^\iagent(\state_0) \in \reals^{m_\iagent}$ and feedback if $\ctrl^\iagent_t = \policy^\iagent(\state_t) \in \reals^{m_\iagent}$.
Policy tuple $\policy$ can constitute different \emph{equilibrium} types of the game, among which the most common ones are Nash~\cite{nash1951non} and Stackelberg~\cite{stackelberg1934}.
In this paper, we do not adhere to a specific information structure or equilibrium type.

\p{Inverse Dynamic Games}
When cost parameters $\theta := (\theta^1, \ldots, \theta^\nagents)$ are initially unknown, we can solve an \textit{inverse} dynamic game~\cite{peters2023online,li2023cost,liu2023learning} to identify these parameters from trajectory data. This is formulated as a \gls{MLE} problem:
\begin{equation}
\begin{aligned}
\label{eq:inverse_game_mle}
    \max_{\theta, \traj, \csig}~\distr(\obstraj \mid \traj, \csig), \quad \text{s.t.} (\traj, \csig) \in \eqset(\theta)
\end{aligned}
\end{equation}
where $\distr(\obstraj \mid \traj, \csig)$ is the \emph{likelihood} of observed trajectory data $\obstraj := \obs_{[0:T]}$ given state trajectories $\traj := \state_{[0:T]}$ and control sequences $\csig := \ctrl_{[0:T-1]}$ that correspond to a (user-specified) equilibrium, and $\eqset(\theta)$ is the set of all such equilibrium solutions parameterized by $\theta$.
\gls{MLE}~\eqref{eq:inverse_game_mle} can handle \textit{corrupted} data: the initial state $\state_0$ is not assumed to be known, and the observation data $\obstraj$ may be noisy and missing certain entries (e.g., at specific time indices).
An inverse game routine such as~\cite{peters2023online,li2023cost,liu2023learning} solves \gls{MLE}~\eqref{eq:inverse_game_mle} by applying gradient ascent to update the value of $\theta$, where gradient $\grad_\theta \distr(\cdot)$ is computed by \textit{differentiating through} a forward game solver that produces an equilibrium solution $(\traj, \csig) \in \eqset(\theta)$.

\p{Nonlinear Opinion Dynamics}
The \gls{NOD} model~\cite{bizyaeva2022nonlinear}
enables fast and flexible multi-agent decision-making.
Consider multi-agent system~\eqref{eq:dyn_sys}, in which each agent $\iagent$ is faced with an arbitrarily large (but finite) number of $\noptioni$ options.
For every $\iagent \in \ordersetagent$ and $\ell \in \ordersetoptioni := \{1,\ldots,\noptioni\}$, define $\opnstate^{\iagent}_\ell \in \reals$ to be the \textit{opinion state} of agent $\iagent$ about option $\ell$.
The more positive (negative) is $\opnstate^{\iagent}_\ell$, the more agent $\iagent$ \textit{favors (disfavors)} option $\ell$.
We say agent $\iagent$ is \textit{neutral} (i.e., undecided) about option $\ell$ if $\opnstate^{\iagent}_\ell=0$.
Compactly, we define $\opnstate^i := (\opnstate^\iagent_{1}, \ldots, \opnstate^\iagent_{\noptioni})$, and $\opnstate := (\opnstate^1, \ldots, \opnstate^{\nagents})$ as the opinion state of agent $i$ and the joint system, respectively.
The evolution of opinion state in continuous time is governed by the \gls{NOD} model adapted from~\cite{bizyaeva2022nonlinear}: 
\begin{equation}
\begin{aligned}
    \label{eq:opn_dyn_orig}
    \dot \opnstate^\iagent &= \opinionDyn^i_c(\opnstate^i) =  - D^{\iagent} \opnstate^{\iagent} + \bias^\iagent + \attstate \saturation^\iagent (\opnstate^{\iagent}),
\end{aligned}
\end{equation}
where the $\ell$-th entry of the saturation term $\saturation^\iagent(\opnstate^\iagent)$ is
\begin{equation*}
\begin{aligned}
    \saturation^{\iagent}_\ell(\cdot) =  &\saturation_1\Big(
        \alpha^{\iagent}_\ell \opnstate^{\iagent}_{\opnidx} +
        \textstyle\sum_{\iagentaux \in \ordersetagent \setminus \{\iagent\} } \gamma^{\iagent \iagentaux}_\ell \opnstate^{\iagentaux}_{\opnidx} \Big) \\
    &+ \sum_{\opnidxaux \in \ordersetthetai \setminus \{\opnidx\} } \saturation_2
    \Big(  
        \beta_{\ell\opnidxaux}^{\iagent} \opnstate^{\iagent}_{\opnidxaux} +
        \textstyle\sum_{\iagentaux \in \ordersetagent \setminus \{\iagent\} } \delta^{\iagent \iagentaux}_{\ell\opnidxaux}  \opnstate^{\iagentaux}_{\opnidxaux}
    \Big),
\end{aligned}
\end{equation*}
which satisfies for $r \in \{1,2\}$, $\saturation_r(0)=0$, $\saturation^{\prime}_r(0)=1$, $\saturation^{\prime\prime}_r(0)\neq 0$, and $\saturation^{\prime\prime\prime}_r(0)\neq 0$.
Valid choices for $\saturation_r(\cdot)$ include the sigmoid function and the hyperbolic tangent function $\tanh$.
In~\eqref{eq:opn_dyn_orig}, $D^{\iagent} = \diag(\damping^\iagent_1,\ldots,\damping^\iagent_\noptioni)$ is the \textit{damping} matrix with each $\damping^\iagent_\ell > 0$, $\bias^\iagent$ represents agent $\iagent$'s own bias, $\attstate > 0$ is the \textit{attention} value on nonlinear opinion exchange, which here taken to be shared across all agents, $\alpha^{\iagent}_\ell \geq 0$ is the self-reinforcement gain, $\beta^{\iagent}_{\ell\opnidxaux} \geq 0$ is the same-agent inter-option coupling gain, $\gamma^{\iagent\iagentaux}_{\ell}$ is the gain of the same-option inter-agent coupling with another agent $\iagentaux$, and $\delta^{\iagent \iagentaux}_{\ell\opnidxaux}$ is the gain of the inter-option inter-agent coupling with another agent $\iagentaux$.
In order to guide a dynamic game with \gls{NOD}, we consider the \textit{discrete-time} version of the \gls{NOD} model jointly for all $\iagent \in \ordersetagent$:
\begin{equation}
\begin{aligned}
    \label{eq:opn_dyn_dt}
    \opnstate_{t+1} &= \opinionDyn(\opnstate_{t}),
\end{aligned}
\end{equation}
which may be obtained by applying time discretization (e.g., forward Euler or Runge-Kutta method) to the continuous-time \gls{NOD} model~\eqref{eq:opn_dyn_orig}.
While \gls{NOD} has demonstrated efficacy in multi-agent decision-making, \emph{automatic} synthesis of its parameters that can change \textit{adaptively} based on physical states, e.g., $\alpha^\iagent_t(\state_t)$, largely remains an open problem.
In the next section, we propose a novel inverse-game-based approach that learns a general \gls{DNN}-parameterized \gls{NOD} model, which allows the opinion state $\opnstate$ to be influenced by physical states $\state$.

\section{Neural NOD for Interactive Robotics}
\label{sec:learn_nod}

In this section, we present our main contribution: Learning Neural \gls{NOD} from inverse dynamic games for fast decision-making in multi-agent non-cooperative interactions.

\subsection{Automatic Tuning of Game Costs using NOD}
Our key insight into opinion states is that they not only model agent $i$'s \textit{attitude} towards option $\ell$ (determined by the sign of $\opnstate^i_\ell$) but also indicate how \textit{strongly} the agent prefers such an option (determined by the magnitude of $\opnstate^i_\ell$).
Therefore, we propose to use an opinion state $\opnstate^i_\ell$ in a dynamic game as the coefficient of the cost term that corresponds to option $\ell$. 
Consequently, we may use \gls{NOD} to evolve the opinion state, \emph{automatically} tuning cost parameter $\theta^i$ in~\eqref{eq:game_cost} such that each agent can rapidly adapt to changes in the environment, a crucial property for interactive robotics.
To this end, we define a class of \emph{opinionated} stage costs in~\eqref{eq:game_cost}, encoding agents' preferences through their opinion states:
\begin{equation}
\label{eq:opn_stage_cost}
   \stagecost_t^\iagent(\state_t, \ctrl_t; \opnstate^\iagent_t) := \bar{\stagecost}_t^\iagent(\state_t, \ctrl_t) + \textstyle\sum_{\ell \in \ordersetoptioni} \opnstate^\iagent_{\ell, t} \stagecost_{\ell, t} ^\iagent(\state_t, \ctrl_t),
\end{equation}
where $\opnstate^\iagent_{\ell, t} := \theta^i_{\ell, t} \in \reals$ is the \emph{cost weight} set to the opinion state of agent $i$ about opinion $\ell$, which is evolved by a state-dependent \gls{NOD} model, $\stagecost_{\ell, t} ^\iagent(\state_t, \ctrl_t) \geq 0$ is the associated \emph{basis} cost term, and $\bar{\stagecost}_t^\iagent(\state_t, \ctrl_t)$ is the \emph{residual} cost term that encodes the remaining task specifications and is independent of the options faced by the agent.
We refer to a dynamic game equipped with the above cost as an \emph{\gls{ODG}}.

\begin{example}
We illustrate our technical approach with a simulated racing example on the 1:1 reconstructed Thunderhill Raceway located in Willows, CA, USA (see Fig.~\ref{fig:full_race}).
In the race, both ego and rival vehicles are constrained to remain within the track boundaries (i.e., at least one wheel is inside the track limit), while only the ego vehicle has the responsibility to avoid a collision.
We model the vehicle motion using the 4D kinematic bicycle model~\cite{zhang2020optimization}.
Similar to~\cite{song2021autonomous}, the ego's basis cost terms (and associated weights) include incentivizing overtaking ($\theta^e_{\rm{ov}}$), following ($\theta^e_{\rm{fl}}$), deviating to the inside ($\theta^e_{\rm{in}}$) or outside ($\theta^e_{\rm{ot}}$); and the rival ones include incentivizing blocking the ego ($\theta^r_{\rm{bl}}$), and deviating to the inside ($\theta^r_{\rm{in}}$) or outside ($\theta^r_{\rm{ot}}$).
The residual cost terms of each player capture optimizing lap time, enforcing safety, and regulating control efforts.
\end{example}

\subsection{Neural Synthesis of NOD from Inverse Dynamic Games}
The key question of using \gls{NOD} in an \gls{ODG}~\eqref{eq:opn_stage_cost} for physical interaction is: \emph{How should \gls{NOD} parameters change adaptively in response to the evolution of physical state $\state$?}
Our central contribution is an inverse game approach to synthesize a \textit{Neural \gls{NOD}} model for \emph{general} \gls{ODG}s with no additional assumptions on each agent's specific task and set of options.
Building on discrete-time \gls{NOD} model~\eqref{eq:opn_dyn_dt},
we may define a \textit{Neural} \gls{NOD} model as: 
\begin{subequations}
\label{eq:neural_nod}
\begin{align}
    \opnstate_{t+1} &= \opinionDyn(\opnstate_{t}; \nodparam_t), && \forall t > 0 \\
    \nodparam_t &= \nnnod(\state_t; \nnparam),  && \forall t > 0 \\
    \opnstate_0 &= \nnnodinit(\state_0; \nnparam_0), \label{eq:neural_nod:z0}
\end{align}
\end{subequations}
where $\nodparam_t = (\nodparam^1_t, \ldots, \nodparam^\nagents_t, \attstate_t)$, with $\nodparam^\iagent_t = (\damping^{\iagent}_t$, $\bias^\iagent_t$, $\alpha^{\iagent}_t$, $\gamma^{\iagent}_t$, $\beta^{\iagent}_t$, $\delta^\iagent_t)$, is the vector that aggregates all \gls{NOD} parameters.
In our proposed scheme, $\nodparam_t$ is predicted by \gls{DNN} $\nnnod$ with parameters $\nnparam$ and input $\state_t$.
Given an initial state $\state_0$, the \gls{NOD} model is initialized with opinion $\opnstate_0$, which is predicted by a separate \gls{DNN} $\nnnodinit$ parameterized by $\nnparam_0$.
The predicted initial opinion state $\opnstate_0$ can be interpreted as the \textit{prior information} of agents' opinions, which we can use to initialize the \gls{NOD} model with a more \textit{informative} opinion than a neutral one (i.e., $\opnstate_0 = 0$) commonly used in the literature.

\begin{figure}[!hbtp]
  \centering
  \includegraphics[width=1.0\columnwidth]{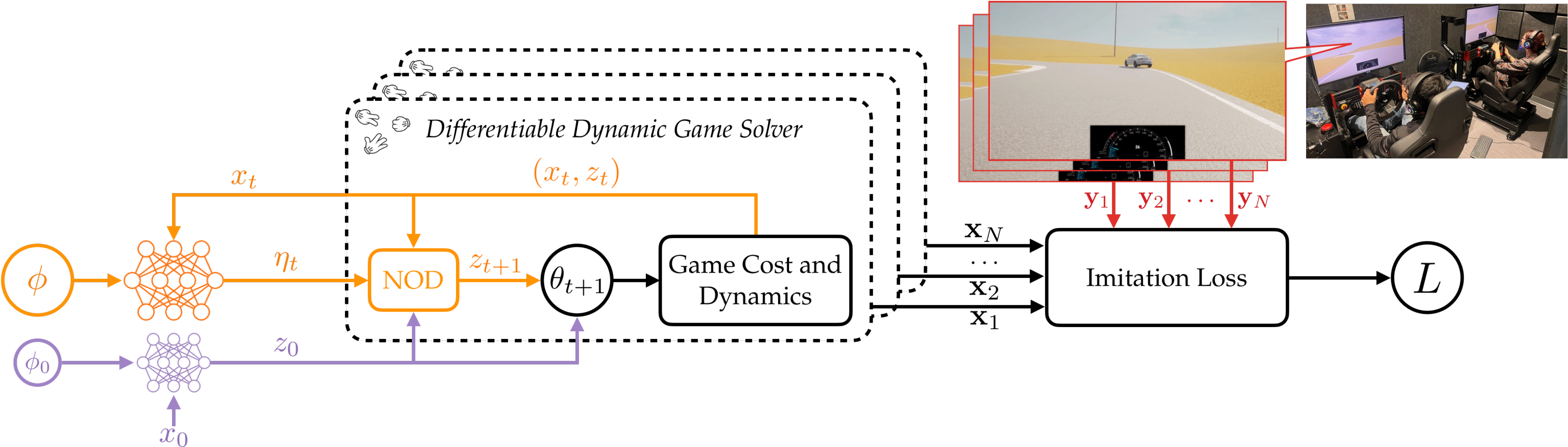}
  \caption{\label{fig:comp_graph} Computation graph of the inverse game for training a Neural \gls{NOD} model illustrated with the autonomous racing example.
  \vspace{-5mm}
  }
\end{figure}

\begin{remark}[\gls{DNN} features]
    The Neural \gls{NOD} framework is \textit{agnostic} to the specific \gls{DNN} feature space or architecture.
    For ease of exposition, we use the current state $\state_t$ as the input feature of \gls{DNN} $\nnnod(\cdot; \nnparam)$.
    Alternatively, we may use an LSTM~\cite{hochreiter1997long} or transformer~\cite{shi2022mtr} model for $\nnnod(\cdot; \nnparam)$, which would then predict \gls{NOD} parameters based on a state history.
\end{remark}

Our technical approach towards synthesizing a Neural \gls{NOD} model is inspired by recent work on inverse dynamic games,
but incorporates several key distinctions.
Existing inverse game approaches, such as~\cite{peters2023online,li2023cost,liu2023learning}, are predominantly limited to learning a \textit{static} cost parameter $\theta$ from a \textit{single} gameplay trajectory data.
As a result, the inverse game problem needs to be solved repeatedly when the environment (e.g., the initial condition or the other agent's intent) changes, presenting computational hurdles for real-time deployment.
In contrast, our proposed approach involves offline training of a Neural \gls{NOD} model from a dataset containing \textit{batches} of gameplay trajectories.
The online computation speed of \textit{agents' policies} via solving an \gls{ODG} with learned Neural \gls{NOD} is close to that of a standard dynamic game, as long as the \gls{DNN} used by Neural \gls{NOD} adopts fast inference.

We cast Neural \gls{NOD} training as an inverse dynamic game.
The corresponding \gls{MLE} problem is formulated as:
\begin{equation}
\begin{aligned}
\label{eq:nod_mle}
    \max_{\nnparam, \nnparam_0, \{(\traj_n, \csig_n)\}_{n\in[N]}}~& \loss := \textstyle\sum_{n=1}^N \distr(\obstraj_n \mid \traj_n, \csig_n)\\
    \text{s.t.}\qquad\quad~& (\traj_n, \csig_n) \in \eqset(\nnparam, \nnparam_0), \quad \forall n \in [N],
\end{aligned}
\end{equation}
where $\loss$ is the \textit{imitation objective} defined by a set of observations $\{\obstraj_n\}_{n \in [N]}$ and the corresponding set of state and control trajectories $\{(\traj_n, \csig_n)\}_{n \in [N]}$, wherein each pair is at an equilibrium of the \gls{ODG} governed by Neural \gls{NOD} with parameter $(\nnparam, \nnparam_0)$.
These parameters can be learned by solving MLE~\eqref{eq:nod_mle} using standard gradient-based methods~\cite{bottou2010large,kingma2014adam}, in which gradients $\nabla_\nnparam \loss$ and $\nabla_{\nnparam_0} \loss$ can be obtained with automatic differentiation (e.g.,~\cite{paszke2017automatic,jax2018github}) by backpropagating through the inverse game computation graph, as shown in~\autoref{fig:comp_graph}.
Note that \gls{MLE}~\eqref{eq:nod_mle} learns \gls{DNN} $\nnnod$ and $\nnnodinit$, which take as input arbitrary state $\state \in \reals^{n_x}$, while \gls{MLE}~\eqref{eq:inverse_game_mle} only works for a specific initial state $\state_0$.
Our proposed training pipeline solely requires that the \gls{ODG} is formulated with any existing \textit{differentiable} dynamic game frameworks, such as open-loop Nash~\cite{peters2020inference}, feedback Nash~\cite{li2023cost}, generalized Nash~\cite{liu2023learning}, stochastic Nash~\cite{mehr2023maximum,lidard2024blending}, open-loop Stackelberg~\cite{hu2024plays}, and feedback Stackelberg~\cite{li2024computation}.
Once a Neural \gls{NOD} model is learned offline from data, it can be deployed plug-and-play in an online \gls{ODG} solver for time-sensitive multi-agent decision-making.

\begin{example}
We use the inverse game approach~\cite{li2023cost} with a negative log-likelihood objective, which yields an approximate feedback Nash equilibrium for the \gls{ODG}.
\end{example}

\subsection{Properties of Neural NOD}
In this section, 
we show that the self-reinforcement gains $\alpha$ and attention $\attstate$ can be adjusted \textit{analytically} so that $\opnstate = 0$ is an \textit{unstable} equilibrium of Neural \gls{NOD}, that is, indecision (i.e., deadlocks) can be broken with an \textit{arbitrarily} small bias term $\bias$.
We start with a technical lemma, which ensures that a pitchfork bifurcation is possible with a Neural \gls{NOD} model.

\begin{lemma}
If there exist $\iagent\in\ordersetagent$ and $\ell\in\ordersetoptioni$ such that $\alpha^\iagent_\ell + \sigma^\iagent_\ell(\bar{\jacobian}_0) > 0$,
where Jacobian matrix $\bar{\jacobian}_0$
\vspace{1mm}
is defined as $\left.\jacobian(\column(\{\bar{\saturation}^\iagent (\opnstate^{\iagent};\nodparam^\iagent)\}_{\iagent \in \ordersetagent} ))\right|_{z=0}$,
$\bar{\saturation}^\iagent (\opnstate^{\iagent};\nodparam^\iagent) := \saturation_1\Big(
\textstyle\sum_{\iagentaux \in \ordersetagent \setminus \{\iagent\} } \gamma^{\iagent \iagentaux}_\ell \opnstate^{\iagentaux}_{\opnidx} \Big) 
+ \sum_{\opnidxaux \in \ordersetthetai \setminus \{\opnidx\} } \saturation_2
\Big(  
\beta_{\ell\opnidxaux}^{\iagent} \opnstate^{\iagent}_{\opnidxaux} +
\textstyle\sum_{\iagentaux \in \ordersetagent \setminus \{\iagent\} }$ $\delta^{\iagent \iagentaux}_{\ell\opnidxaux}  \opnstate^{\iagentaux}_{\opnidxaux}
\Big)$,
and $\sigma^\iagent_\ell(\bar{\jacobian}_0)$ is the eigenvalue of $\bar{\jacobian}_0$ at the same location in the spectral matrix of $\bar{\jacobian}_0$ as $\alpha^\iagent_\ell$ in diagonal matrix $\mathcal{A} = \diag\{\alpha^i_{\ell}\}_{i\in\ordersetagent,\ell\in\ordersetoptioni}$,
then Jacobian matrix $\jacobian_0 := \left.\jacobian(\column(\{\saturation^\iagent (\opnstate^{\iagent};\nodparam^\iagent)\}_{\iagent \in \ordersetagent} ))\right|_{z=0}$ has at least one eigenvalue with a positive real part, i.e., $\real^+(\sigma(\jacobian_0)) \neq \emptyset$.
\end{lemma}

\begin{theorem}[Guaranteed Indecision Breaking]
\label{thm:indecision_breaking}
    Consider the continuous-time Neural \gls{NOD} model $\dot{\opnstate} = \opinionDyn_c(\opnstate; \nodparam)$, where $\opinionDyn_c = \column(\opinionDyn_c^1,\ldots,\opinionDyn_c^\nagents)$ with $\opinionDyn_c^i(\cdot; \nodparam^i)$ defined in~\eqref{eq:opn_dyn_orig} and $\nodparam^i$ defined in~\eqref{eq:neural_nod}.
    If $\exists \iagent\in\ordersetagent$, $\ell\in\ordersetoptioni$ such that $\alpha^\iagent_\ell + \sigma^\iagent_\ell(\bar{\jacobian}_0) > 0$, then the following results hold:
    \begin{itemize}
        \item When bias term $\bias = 0$, if $\attstate > \|\column(\{\damping_\iagent\}_{\iagent \in \ordersetagent})\|_{\infty}/$ $\max\real^+(\sigma(\jacobian_0))$, then the neutral opinion $\opnstate=0$ is an unstable equilibrium of the Neural \gls{NOD}.
        Moreover, the opinion state departs $\opnstate=0$ (locally) at an exponential rate with an arbitrarily small bias~$\bias$,
        \item When bias term $\bias \neq 0$, the neutral opinion $\opnstate=0$ is not an equilibrium of the Neural \gls{NOD} and the pitchfork unfolds, i.e., the model is ultrasensitive at $\opnstate=0$.
    \end{itemize}
\end{theorem}

\begin{proof}
    Linearizing Neural \gls{NOD} at neutral opinion $\opnstate = 0$ gives $\dot{\opnstate} = (-D + \jacobian_0) \opnstate + \bias$, where $D = \blkdiag(\{D^i\}_{i \in \ordersetagent})$.
    When $\bias = 0$, indecision $\opnstate=0$ is an equilibrium.
    Since $\real^+(\sigma(\jacobian_0)) \neq \emptyset$, if the attention variable $\attstate$ is chosen such that $\attstate > \attstate^* = \|\column(\{\damping_\iagent\}_{\iagent \in \ordersetagent})\|_{\infty} / \max\real^+(\sigma(\jacobian_0))$, then matrix $-D + \jacobian_0$ has one eigenvalue with positive real part, i.e., indecision $\opnstate=0$ is exponentially unstable.
    Thus, an indecision-breaking pitchfork bifurcation occurs at critical attention
    value $\attstate^*$.
    When $\bias \neq 0$, the pitchfork bifurcation unfolds, as predicted by the unfolding theory~\cite{golubitsky1985singularities}.
\end{proof}

\section{Simulation Results}
\label{sec:results}

We use our proposed Neural \gls{NOD} model, trained on both synthetic and human datasets, to automatically tune the game cost weights within the ILQGame planning framework~\cite{fridovich2020efficient}.
We evaluate its performance in simulated autonomous racing scenarios (our running example).
For both planning and simulation, we employ the 4D kinematic bicycle model~\cite{zhang2020optimization}, discretized with a time step of $\Delta t = 0.1$ s, to describe the vehicle motion.
We implement the inverse game training pipeline with Flax~\cite{flax2020github} and train all neural networks with Adam~\cite{kingma2014adam}.
All planners are implemented using JAX~\cite{jax2018github} and run in real-time at a frequency of 10Hz on a desktop with an AMD Ryzen 9 7950X CPU.

\p{Hypotheses}
We make three hypotheses that showcase the strengths of Neural \gls{NOD}.
\begin{itemize}
    \item \textbf{H1 (Performance).} \emph{The Neural \gls{NOD} model leads to safer and more competitive robot motion.}
    
    \item \textbf{H2 (Generalization).} \emph{The Neural \gls{NOD} model generalizes better to out-of-distribution rival behaviors.}
    
    \item \textbf{H3 (Human data compatibility).} \emph{The inverse game training pipeline can effectively learn a Neural \gls{NOD} model from noisy human data.}
\end{itemize}

\p{Baselines}
We compare our approach with two baseline methods.
All neural networks have five fully-connected layers with 256 neurons each. The network features include all vehicles' physical states.
\begin{itemize}
    \item \textbf{Multi-layer perceptron learned from an inverse game (MLP-IG).} A mapping from feature $\state_t$ to cost weights $\theta_t$, parameterized by a \gls{MLP} model and learned with an inverse dynamic game.
    This approach is briefly explored in the prior state-of-the-art inverse game approach~\cite{liu2023learning}, which learns a \gls{DNN} that predicts game objectives.
    \item \textbf{\gls{E2E-BC}.} An end-to-end control policy (i.e., one that takes as input state $\state_t$ and returns agents' control $\ctrl_t$) learned with behavior cloning and supervised learning~\cite{bojarski2016end}.
\end{itemize}

\p{Metrics} We consider the following performance metrics:
\begin{itemize}
    \item \textbf{Safe rate (SR).} A ratio defined as $N_{\text{safe}} / N_{\text{trial}} \times 100\%$, where $N_{\text{safe}}$ is the number of safe trials---those in which the ego car stays within the track limits and avoids collisions with its rival at all times---and $N_{\text{trial}}$ is the total number of trials.
    \item \textbf{Overtaking rate (OR).} A ratio defined as $N_{\text{overtake}} /$ $N_{\text{trial}} \times 100\%$, where $N_{\text{overtake}}$ is the number of safe trials in which the ego car successfully overtakes the rival \emph{and} maintains a lead over it at the end of the trial.
    \item \textbf{Average end-time leading distance (AELD).} The distance between the ego car and its rival at the end of a trial, measured in meters.
\end{itemize}
We will report SR and OR in percentage, and AELD with mean and standard deviation calculated across all trials.

\subsection{Synthetic dataset}
In the first experiment, we learn all policies from a synthetic dataset with racing demonstrations obtained with a model-based trajectory optimizer.
We first evaluated each policy with in-distribution rival behaviors, where the rival used an ILQGame policy whose blocking cost weight was randomized within the same range ($\theta^r_{\rm{bl}} \in [0,30]$) as the training data. 
We simulated 100 races with randomized initial states and rival cost weights, and the statistics are shown in \autoref{tab:synth_id_results}.
The proposed Neural \gls{NOD} outperformed both baselines in all metrics.

\begin{figure}[!hbtp]
  \centering
  \includegraphics[width=0.92\columnwidth]{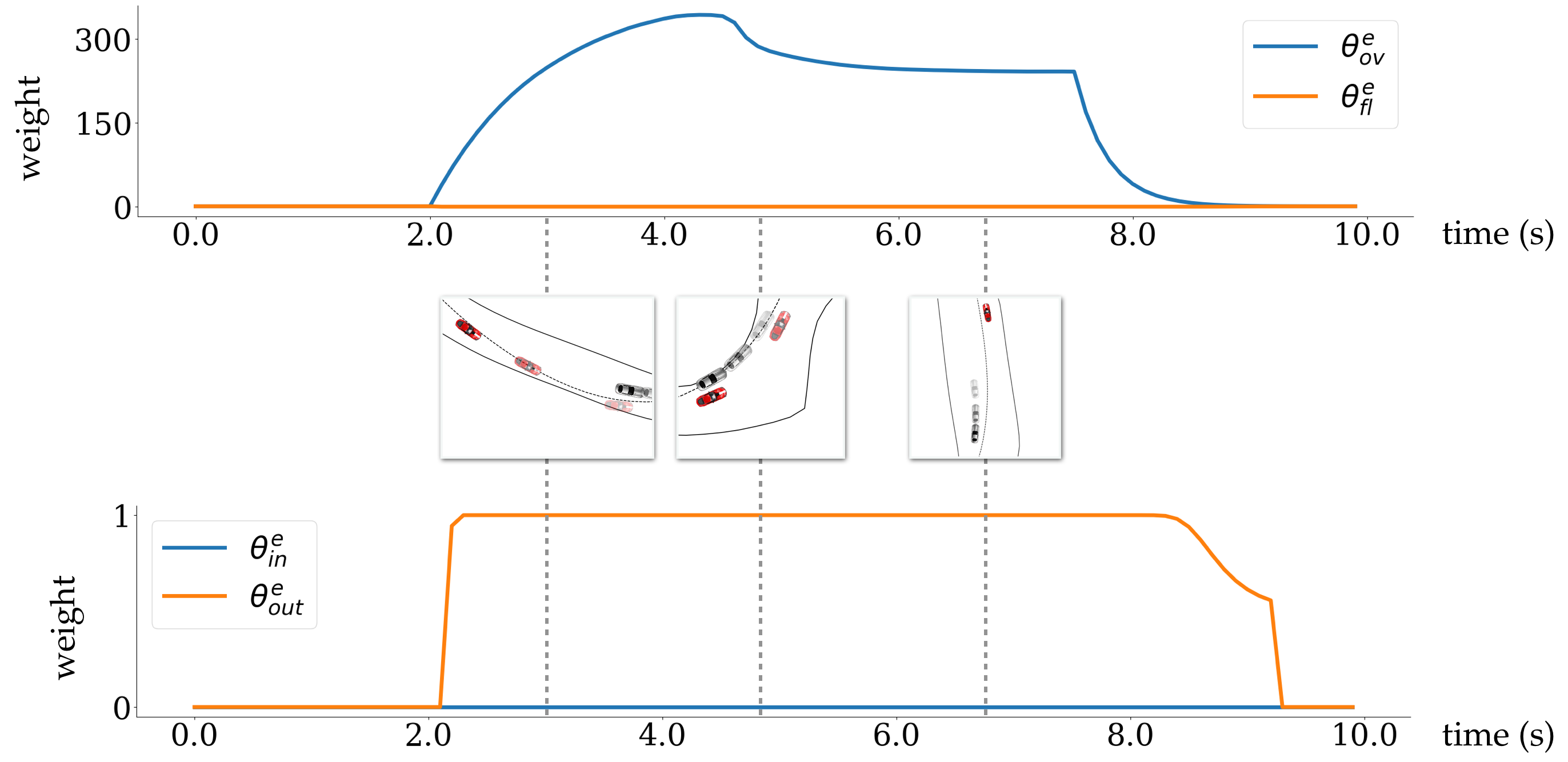}
  \caption{\label{fig:traj_nod} Simulation snapshots and the time evolution of game cost weights, when the ego vehicle uses the Neural \gls{NOD} learned from the synthetic dataset. Planned future motions are displayed with transparency. The racing line is plotted in dashed black.
  The ego car made a timely decision to speed up and move to the outside, safely overtaking the rival.
  }
  \vspace{-7mm}
\end{figure}

We then simulated another 100 randomized races against a more aggressive rival, whose blocking cost weight was randomized within the range $\theta^r_{\rm{bl}} \in [60,80]$. 
As shown in \autoref{tab:synth_ood_results}, in this more challenging, out-of-distribution setting, Neural \gls{NOD} led the baselines by a significant margin in terms of safe rate, overtaking rate, and AELD.
The \gls{E2E-BC} baseline yielded a significantly lower safe rate, confirming the well-known generalization issue of behavior cloned policies when deployed in previously unseen scenarios.

Finally, in 100 additional randomized races where the rival used the \gls{E2E-BC} policy,
the Neural \gls{NOD} continued to outperform the baselines in all metrics.
Those results jointly validated H1.
Results in \autoref{tab:synth_ood_results} and \autoref{tab:synth_bc_results} validated H2.

\begin{remark}[Information Privilege]
    Although rival's cost weights were not accessible to the ego,
    the solvers still shared a set of common knowledge (e.g., the basis cost functions), which can constitute privileged information of the game-based policies over \gls{E2E-BC}.
    In \autoref{tab:synth_bc_results}, we provide results with the rival using the \gls{E2E-BC} policy, in which case the game-based policies no longer have an information privilege.
\end{remark}

\setlength\tabcolsep{10pt}
\begin{table}[!hbtp]
\centering
\resizebox{\columnwidth}{!}{
\begin{tabular}{l|ccc}
\toprule
Method & 
SR [$\%$] $\uparrow$ &  
OR [$\%$] $\uparrow$ &  
AELD [m] $\uparrow$
\\ \midrule
\gls{MLP}-IG         &$\mathbf{95\%}$   &$93.68\%$   &$42.87 \pm 24.30$          \\
\gls{E2E-BC}         &$90\%$   &$88.89\%$   &$39.39 \pm 26.68$          \\
\emph{Neural \gls{NOD} (ours)}  &$\mathbf{95\%}$    &$\mathbf{96.84\%}$  &$\mathbf{50.85 \pm 29.91}$   \\ \bottomrule
\end{tabular}
}
\vspace{0.1em}
\caption{Results obtained from 100 randomized races with in-distribution rival behaviors and policies trained with synthetic data.}
\label{tab:synth_id_results}
\vspace{-6mm}
\end{table}

\setlength\tabcolsep{10pt}
\begin{table}[!hbtp]
\centering
\resizebox{\columnwidth}{!}{
\begin{tabular}{l|cccc}
\toprule
Method & 
SR [$\%$] $\uparrow$ &  
OR [$\%$] $\uparrow$ &  
AELD [m] $\uparrow$
\\ \midrule
\gls{MLP}-IG         &$82\%$   &$62.20\%$   &$4.50 \pm 37.84$         \\
\gls{E2E-BC}         &$66\%$   &$60.61\%$   &$3.64 \pm 33.71$         \\
\emph{Neural \gls{NOD} (ours)}  &$\mathbf{91\%}$    &$\mathbf{76.92\%}$  &$\mathbf{7.60 \pm 42.25}$  \\ \bottomrule
\end{tabular}
}
\vspace{0.1em}
\caption{Results obtained from 100 randomized races in an out-of-distribution evaluation with more aggressive rival behaviors.}
\vspace{-6mm}
\label{tab:synth_ood_results}
\end{table}

\setlength\tabcolsep{10pt}
\begin{table}[!hbtp]
\centering
\resizebox{\columnwidth}{!}{
\begin{tabular}{l|cccc}
\toprule
Method & 
SR [$\%$] $\uparrow$ &  
OR [$\%$] $\uparrow$ &  
AELD [m] $\uparrow$ 
\\ \midrule
\gls{MLP}-IG         &$70\%$   &$61.43\%$   &$11.77 \pm 35.00$          \\
\gls{E2E-BC}         &$55\%$   &$61.82\%$   &$14.32 \pm 35.07$          \\
\emph{Neural \gls{NOD} (ours)}  &$\mathbf{87\%}$    &$\mathbf{74.71\%}$  &$\mathbf{16.34 \pm 52.49}$   \\ \bottomrule
\end{tabular}
}
\vspace{0.1em}
\caption{Results obtained from 100 randomized races with the rival using the behavior-cloned policy trained with synthetic data.}
\label{tab:synth_bc_results}
\vspace{-6mm}
\end{table}

\setlength\tabcolsep{10pt}
\begin{table}[!hbtp]
\centering
\resizebox{\columnwidth}{!}{
\begin{tabular}{l|cccc}
\toprule
Method & 
SR [$\%$] $\uparrow$ &  
OR [$\%$] $\uparrow$ &  
AELD [m] $\uparrow$ 
\\ \midrule
\gls{MLP}-IG                    &$78\%$             &$53.85\%$   &$14.49 \pm 42.95$          \\
\gls{E2E-BC}                    &$62\%$             &$75.81\%$   &$12.35 \pm 34.77$          \\
\emph{Neural \gls{NOD} (ours)}  &$\mathbf{81\%}$    &$\mathbf{82.72\%}$  &$\mathbf{15.94 \pm 34.55}$   \\ \bottomrule
\end{tabular}
}
\vspace{0.1em}
\caption{Results obtained from 100 randomized races with ego policy trained on the human-generated dataset. The rival uses a game policy with randomized cost weights.}
\label{tab:human_data_results}
\end{table}

\begin{figure}[!hbtp]
  \centering
  \includegraphics[width=1.0\columnwidth]{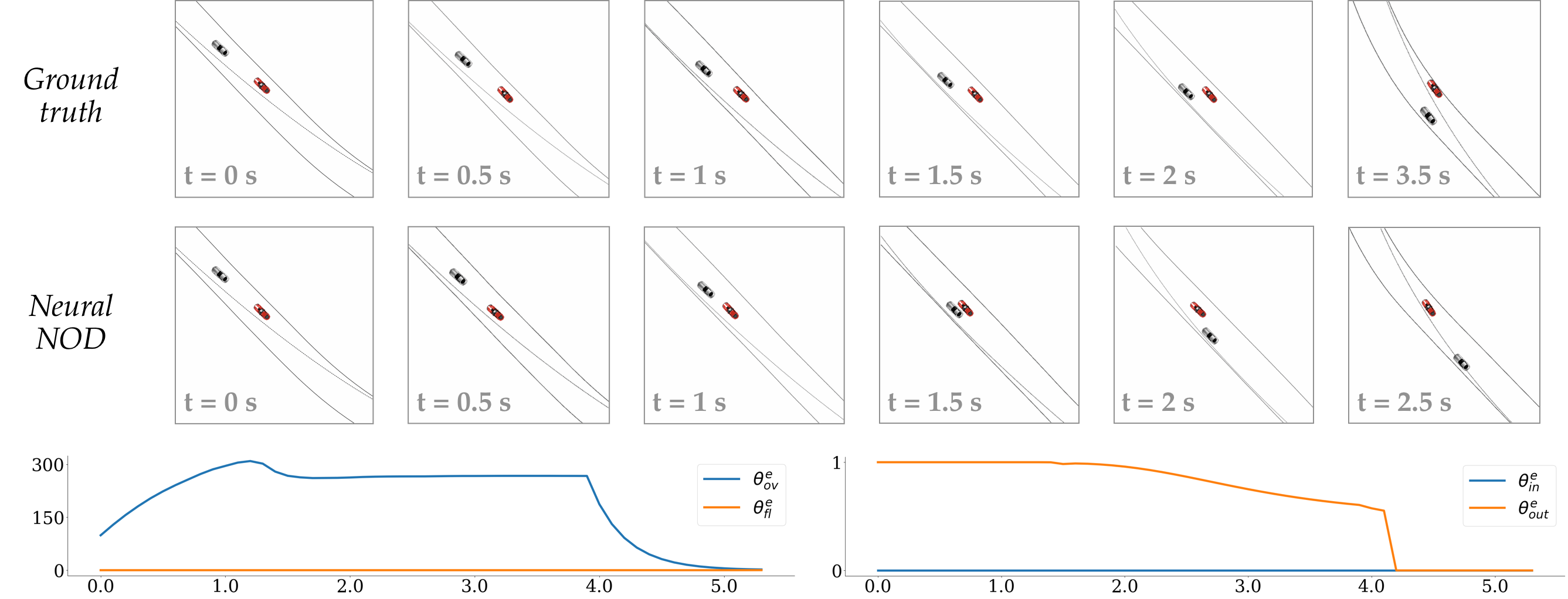}
  \caption{\label{fig:human_gt} Comparing a simulated gameplay trajectory against the groundtruth.
  \emph{Top:} Groundtruth trajectories of the ego and rival.
  \emph{Middle:} Simulation snapshots when the ego vehicle uses the Neural \gls{NOD} to race against a rival, whose motion is replayed from the groundtruth data.
  \emph{Bottom:} Time evolution of game cost weights tuned by the Neural \gls{NOD}.
  }
  \vspace{-6mm}
\end{figure}

In \autoref{fig:traj_nod}, we examine one representative simulated race with the ego using the Neural \gls{NOD}.
The vehicles started from the same initial conditions, and the rival used a game policy whose cost parameters were not accessible to the ego's policy.
The ego agent rapidly formed a strong opinion in favor of overtaking the rival from the outside.
This resulted in a smooth, decisive, and safe overtaking maneuver.

\subsection{Human-generated dataset}

Next, we learn all policies from a dataset with racing demonstrations performed by human drivers.
To obtain such racing data, we leveraged a driving simulator (shown in the upper right corner of~\autoref{fig:comp_graph}), where the ego car was driven by human participants of various skill levels against a rival using a reactive \gls{RL}-based policy in CARLA~\cite{dosovitskiy2017carla}.
For evaluation, we simulated 100 races with randomized initial states and a rival using a game policy with randomized cost weights inaccessible to the ego's planner.
The racing statistics are shown in \autoref{tab:human_data_results}.
The Neural \gls{NOD} again outperformed both baselines in all metrics.
This result coincides with those obtained based on the synthetic dataset, confirming the ability of our proposed inverse game framework to learn a Neural \gls{NOD} from noisy human data (H3).
In \autoref{fig:human_gt}, we compare a simulation trajectory generated by Neural \gls{NOD} with the groundtruth replay.
Under this initial condition, when the rival tried to defend from the inside (left), the Neural \gls{NOD} decided to attack from the outside (right), controlling the ego car to successfully overtake the rival.
Note that this decision closely resembles the groundtruth human demonstration, also displayed in \autoref{fig:human_gt}.

\begin{figure}[!hbtp]
  \centering
  \includegraphics[width=1.0\columnwidth]{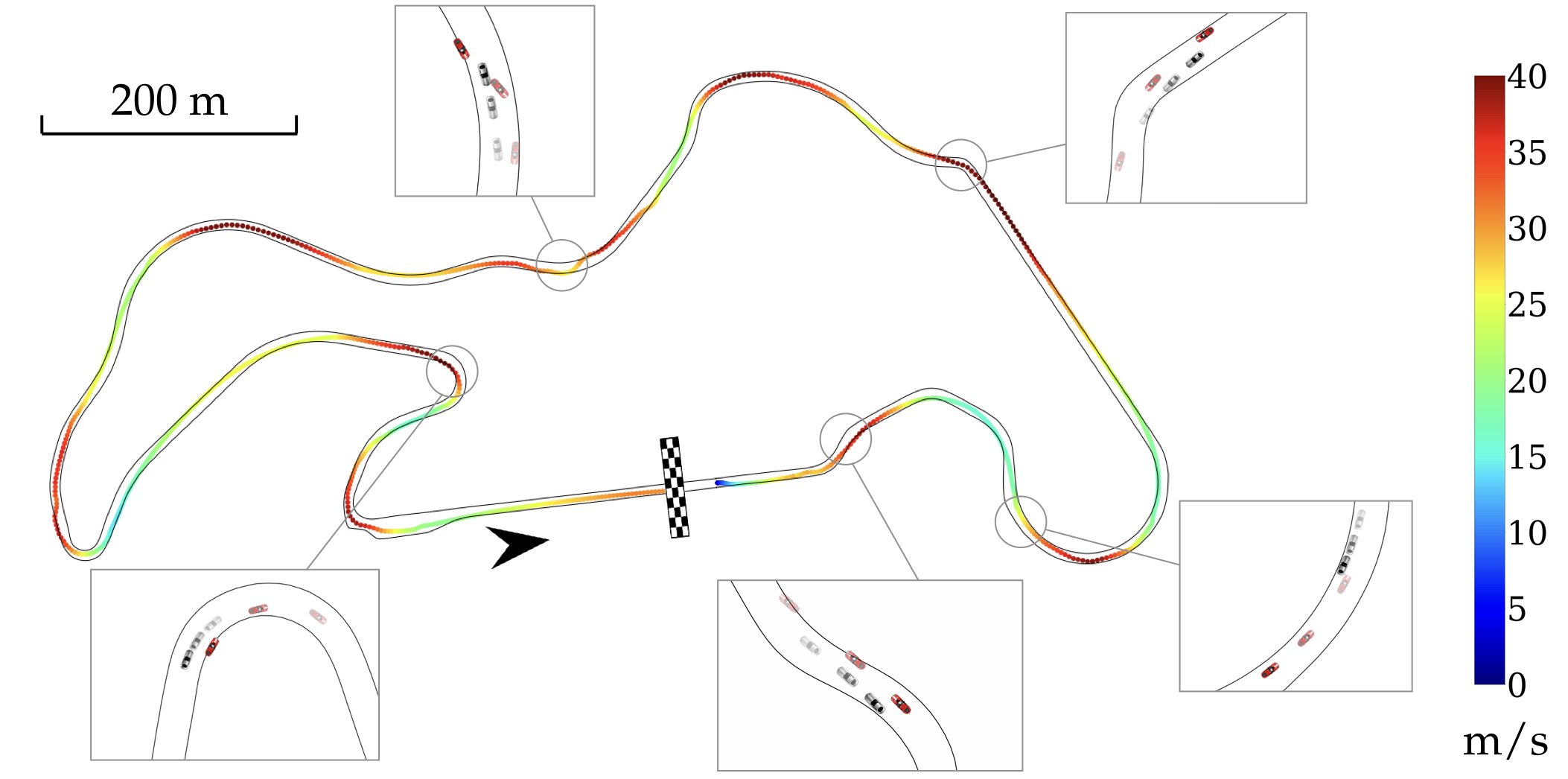}
  \caption{\label{fig:full_race} Ego trajectory and velocity profile of the full endurance race at the Thunderhill Raceway. Transparent footprints denote the planned motion of the ego (red) and rival (silver).
  The black arrow and the grid indicate the track direction and finish line, respectively.
  \vspace{-8mm}
  }
\end{figure}

\subsection{The full endurance race}
In the last example, we test Neural \gls{NOD} in an \emph{endurance race} inspired by the popular video game Real Racing 3.
In this race, the ego car is required to complete one full lap with dual objectives: overtaking \textit{as many rivals} as possible while minimizing the lap time.
Each time a rival player is overtaken, it will be eliminated from the race, and a new rival will be initialized in front of the ego car. 
We use \gls{DNN} $\nnnodinit$ (c.f.~\eqref{eq:neural_nod:z0}) to reset the opinion states when the rival is respawned.
Using the Neural \gls{NOD}, the ego completed the race in $101.4$ s and overtook $9$ rivals without collision.
With the \gls{MLP} baseline, the ego finished the race in $114.2$ s, overtaking $7$ rivals but incurred $2$ crashes.
Finally, under the behavior cloning baseline, the race was concluded in $110.1$ s, during which the ego overtook $9$ rivals but with $3$ crashes.
This result validated H1.

\section{Limitations and Future Work}
\label{sec:future_work}

In this paper, we consider a restricted setting where Neural \gls{NOD} is used to automatically tune the (continuous) game cost weights.
As shown in recent work~\cite{hu2023emergent}, an opinion state can be interpreted as a simplex.
Therefore, a Neural \gls{NOD} model is also capable of coordinating agents over \textit{integer-valued} options such as leadership assignment in a Stackelberg setting~\cite{khan2023leadership,hu2024plays}.
In addition, while attention $\lambda$ governed by the black-box \gls{DNN} generally performs well for car racing, 
we may use an explicit attention model to enable \textit{excitable} decision-making by, for example, adding an extra slower negative feedback loop~\cite[Sec.~7.2.2]{leonard2024fast}.
Such enhanced flexibility would allow agents to \textit{forget} prior decisions to prevent sticking to decisions that may no longer be safe and/or competitive as the environment evolves rapidly.
Finally, we see an open opportunity to use Neural \gls{NOD} for \emph{shared autonomy}, e.g., AI-assisted car racing~\cite{TANIGUCHI2014, decastrodreaming}, in which the human and robot \emph{simultaneously} provide control inputs to the system while interacting with other agents. 
By modeling the agents' intents as opinion states and planning robot motion in the joint intent--physical space~\cite{hu2024doxo}, we can achieve not only \textit{value alignment}, i.e., the robot infers and adopts the human's goals, but also \textit{automation transparency}, i.e., the human is aware of the robot’s current intent, both rapidly and decisively with Neural \gls{NOD}.

\section{Conclusion}
\label{sec:conclusion}
In this paper, we introduced a Neural \gls{NOD} model for game-theoretic robot motion planning in a split-second, and an inverse game approach to learn such a model from data. 
We also provided a constructive procedure to adjust Neural \gls{NOD} parameters online such that breaking of indecision is guaranteed.
Through extensive simulation studies of autonomous racing based on real-world circuit and human-generated interaction data, we demonstrated that a dynamic game policy guided by a Neural \gls{NOD} can consistently outperform state-of-the-art imitation learning and data-driven inverse game policies.

\balance
\printbibliography{}

\end{document}